\documentclass[11pt]{article}

\usepackage[english]{babel}
\usepackage[a4paper, margin=2.5cm]{geometry}
\usepackage{amsthm}
\usepackage{amsmath}
\usepackage{amssymb}
\usepackage{array}
\usepackage{booktabs}
\usepackage{amsfonts}
\usepackage{siunitx}
\usepackage{xcolor}
\usepackage{tikz, pgfplots}
\usepackage{float}
\usepackage[font=small,labelfont=bf]{caption}

\newcolumntype{L}[1]{>{\raggedright\let\newline\\\arraybackslash\hspace{0pt}}m{#1}}
\newcolumntype{C}[1]{>{\centering\let\newline\\\arraybackslash\hspace{0pt}}m{#1}}
\newcolumntype{R}[1]{>{\raggedleft\let\newline\\\arraybackslash\hspace{0pt}}m{#1}}

\allowdisplaybreaks
\pgfplotsset{compat=1.14}

\newtheorem{theorem}{Theorem}

\title{$k$-RNN: Extending NN-heuristics for the TSP}

\author{Nikolas Klug, Alok Chauhan, Ramesh Ragala, V Vijayakumar}

\begin{document}
	\maketitle
	\begin{abstract}
		In this paper we present an extension of existing Nearest-Neighbor heuristics to an algorithm called $k$-Repetitive-Nearest-Neighbor. 
		The idea is to start with a tour of $k$ nodes and then perform a Nearest-Neighbor search from there on. 
		After doing this for all permutations of $k$ nodes the result gets selected as the shortest tour found. 
		Experimental results show that for $2$-RNN the solutions quality remains relatively stable between about $10\%$ to $40\%$ above the optimum.
	\end{abstract}

	\section{Introduction}
	\label{sec:introduction}
	The traveling salesman problem (TSP) is one of the most studied problems in theoretical computer science. 
	Despite its simplicity, no efficient algorithms exist. 
	The problem is, in fact, NP-hard and asks to find a shortest Hamiltonian Cycle in a given graph -- a permutation of vertices for which the sum of the connecting edges is the least.
	We refer to the TSP in undirected complete graphs as Symmetric TSP (STSP) and in directed complete graphs as Asymmetric TSP (ATSP).
	Following a popular convention for TSP-related algorithms, we will also refer to a Hamiltonian cycle as a tour. 	

	Due to the problem's popularity and its broad area of application, many heuristics and approximation algorithms have emerged. One of the popular ones are the $2$-Opt heuristic $[$2$]$, the $3$-Opt algorithm $[$4$]$ and the $k$-opt extension by Lin and Kernigham $[$5$]$ which uses a dynamic $k$ value.
	
	The idea behind the algorithm which is presented here is the "Nearest-Neighbor" heuristic (NN).
	It has already been mentioned in the 1960s by Bellmore and Nemhauser $[$1$]$.
	The basic idea of this algorithm is to pick one starting node randomly and repeatedly extend the sub-tour by its current nearest neighbor until a full tour is formed.
	
	$k$-Repetitive-Nearest-Neighbor ($k$-RNN), the algorithm presented here, provides some abstraction and generalization to this Nearest-Neighbor heuristic. 
	It is applicable to both STSP and ATSP and the experimental results in Section \ref{sec:experimental} do not show a big difference in quality of solution between those two.
	
	The original idea for this extension of Nearest-Neighbor heuristics was what now has become $2$-RNN. It was inspired by an Indian philosophy called Madhyasth Darshan $[$6,7$]$. 
	
	In Section \ref{sec:algorithm} we present the algorithm and show some of its properties, then we show some related work. 
	In Section \ref{sec:experimental} we show experimental results for $1$- and $2$-RNN obtained from running some samples form TSPLIB $[$8$]$.
	 
	\section{Algorithm}
	\label{sec:algorithm}
	In this section we will present the family of algorithms we call $k$-Repetitive-Nearest-Neighbor ($k$-RNN) algorithms. 
This abstracts the Nearest-Neighbor (NN) and Repetitive-Nearest-Neighbor (RNN) heuristics and extend them to a more general basis.

Let $G = (V, E)$ be a complete graph and $k \in \mathbb{N}$. Let $v_1, v_2, \mathellipsis, v_k$ be distinct vertices of $G$. 
Let $c: V \times V \rightarrow \mathbb{R}$ be a cost function describing the cost of the edge between two vertices. 
The algorithm consists of the following steps:

\begin{description}
	\item[Step 1:] For every combination of the $k$ vertices $v_1, v_2, \mathellipsis, v_k$ create the partial tour $T = (v_1, v_2, \mathellipsis, v_k)$ and mark the vertices $v_1, v_2, \mathellipsis, v_k$ as visited.
	
	\item[Step 2:] Set $i = k$. While there are unvisited vertices left: 
	Select $v_{i+1}$ as the nearest unvisited neighbor of $v_i$ and append $v_{i+1}$ to $T$. 
	If there are multiple nearest neighbors, select any.
	Mark $v_{i+1}$ as visited and increment $i$ by $1$.
	
	\item[Step 3:] Among all $\frac{n!}{(n-k)!}$ tours found select the shortest as the result
\end{description}	

Note that for $k = 1$ the algorithm equals the well-known Repetitive-Nearest-Neighbor approach, and one might define the case $k = 0$ as the popular Nearest-Neighbor algorithm, where one starting node gets selected randomly.

The running time of the algorithm consists of two major parts. 
The first is the outer for-loop in Step 1. 
Since for every $k \in \mathbb{N}$ there exist $\frac{n!}{(n-k)!}$ (ordered) combinations of nodes and since $\frac{n!}{(n-k)!} \in O(n^k)$, the for-loop in Step 1 does a total of $O(n^k)$ iterations.
The second important part for the running time is Step 2. In order to find the nearest unvisited neighbor, the algorithm considers a maximum of $n$ edges emerging from the current node. 
Since this is done for $n - k = O(n)$ nodes, the total running time of Step 2 is $O(n^2)$.
Therefore a $k$-RNN run takes time $O(n^2 \cdot n^k) = O(n^{k+2})$.

Similar to the $k$-Opt algorithms $[$2,5$]$, for $k = n$, $k$-RNN also gives the exact solution. In this case, the algorithm degrades to a plain brute-force search.
Due to the running time of $O(n!)$ this is not the preferable way to obtain an exact solution.

In the following we refer to an execution of $k$-RNN for a specific instance and fixed $k$ as a "run" of the algorithm.
We now present a property of the algorithm about the quality of the solution found by $k$-RNN.
\begin{theorem}
	\label{theo:quality}
	For $k < l$, if there exist no vertices $a, b, c$ with $c(a, b) = c(a, c)$, the result of an $l$-RNN run is always better or equal to the result of a $k$-RNN run.
\end{theorem}
\begin{proof}
	Let $T$ be a tour found by a $k$-RNN run and $(v_1, v_2, \mathellipsis, v_l)$ the first $l$ nodes of $T$. 
	A $l$-RNN run considers every permutation of $l$ nodes as a starting tour, so especially the permutation $(p_1, p_2, \mathellipsis, p_l)$ where $p_i = v_i$ for $1 \leq i \leq l$. 
	From there on, both runs construct the same tour.
\end{proof}

There are also some variations of the presented algorithm of which we want to mention two.

In case of multiple nearest neighbors in Step 2, one node gets chosen randomly. 
In order to avoid this non-deterministic choice, a variation of the algorithm constructs multiple paths from there on, one for each nearest neighbor. 
Although this eliminates the randomness, in some cases this will make the running time exponential.

Another approach is the following:
In Step 2 node $v_{i+1}$ gets selected as the nearest unvisited neighbor of the last node of the current partial tour. 
This can be extended in such a way that the new node can also be selected as the nearest unvisited neighbor of the first node of the partial tour. 
We refer to this approach as bi-directional $k$-RNN (Bi-$k$-RNN). Step 2 of the algorithm would then look the following:

\begin{description}
	\item[Step 2:] Set $v_e = v_k$ and $v_s = v_1$. 
	While there are unvisited vertices left: 
	Select $q$ as 
	\[
	\arg \min\{c(v_e, p),\ c(v_s, p) |\ p \in V \text{ and $p$ is unvisited}\}
	\] 
	and mark $q$ as visited. If $c(q, v_s) < c (v_e, q)$, insert $q$ at the start of $T$ and set $v_s = q$. Else append $q$ to the end of $T$ and set $v_e = q$.
\end{description}

We also present results of this variation in Section \ref{sec:experimental}.

	\section{Related Work}
	\label{sec:related}
	
	Similar to Gutin, Yeo and Zverovich in their domination analysis of greedy heuristics for the TSP $[$3$]$, for every $n \in \mathbb{N}$ we define the \textit{domination number} of an algorithm for the TSP as the maximum integer $d(n)$ for which the algorithm produces a tour that is better or equal than $d(n)$ other tours.

	In general the Nearest-Neighbor approach has been shown to be sub optimal $[$3$]$: 
	For each number of nodes there exist some instances for which the algorithm produces a very poor result.
	In fact, the domination number for $0$-RNN is 1, the worst possible domination number.
	$1$-RNN has a domination number of at most $n - 1$ and at least $n / 2$.
	As shown in theorem \ref{theo:quality}, every run of the algorithm for $k \geq 1$ considers all tours a $1$-RNN run for the same instance does.
	Therefore for $k \geq 1$, $k$-RNN also has a domination number of at least $n / 2$.
	
	Despite this all samples tested in our experiments in Section \ref{sec:experimental} produce reasonable results.
	\section{Experimental Results}
	\label{sec:experimental}
	
	In the following we are going to present some experimental results of the algorithm. The experiments where conducted for several instances of TSPLIB $[$8$]$, symmetric as well as asymmetric, comparing $1$- and $2$-RNN runs.
Because of the running time, larger $k$-values could not be tested adequately.

\begin{figure}[!htbp]
	\centering
	\begin{minipage}{\linewidth}
		\centering
		\scriptsize
		\begin{tabular}{*{3}{L{1.6cm}} >{\bfseries}L{1.6cm} L{1.6cm} >{\bfseries}L{1.6cm} L{1.6cm} >{\bfseries}L{1.6cm}}
			\toprule
			&& \multicolumn{2}{c}{1-RNN}& \multicolumn{2}{c}{2-RNN} & \multicolumn{2}{c}{Bi-2-RNN} \\
			\cmidrule(lr){3-4}
			\cmidrule(lr){5-6}
			\cmidrule(lr){7-8}
			Dataset & Optimum & Result & Excess & Result & Excess & Result & Excess \\
			\midrule
			a280      & 2579   & 2975   & 15.35  & 2953   & 14.50 & 2951   & 14.42  \\
			berlin52  & 7542   & 8181   & 8.47   & 7968   & 5.65  & 8380   & 11.11  \\
			bier127   & 118282 & 133953 & 13.25  & 128589 & 8.71  & 129133 & 9.17   \\
			brazil58  & 25395  & 27384  & 7.83   & 27213  & 7.16  & 27115  & 6.77   \\
			brg180    & 1950   & 8890   & 355.90 & 2020   & 3.59  & 8890   & 355.90 \\
			ch130     & 6110   & 7129   & 16.68  & 6903   & 12.98 & 6833   & 11.83  \\
			ch150     & 6528   & 7113   & 8.96   & 7113   & 8.96  & 7075   & 8.38   \\
			d1291     & 50801  & 58681  & 15.51  & 58681  & 15.51 & 58460  & 15.08  \\
			d1655     & 62128  & 73369  & 18.09  & 72554  & 16.78 & 71858  & 15.66  \\
			d198      & 15780  & 17620  & 11.66  & 17405  & 10.30 & 17753  & 12.50  \\
			d493      & 35002  & 40186  & 14.81  & 40186  & 14.81 & 39821  & 13.77  \\
			d657      & 48912  & 60174  & 23.03  & 59310  & 21.26 & 58874  & 20.37  \\
			dantzig42 & 699    & 864    & 23.61  & 826    & 18.17 & 848    & 21.32  \\
			eil101    & 629    & 746    & 18.60  & 743    & 18.12 & 738    & 17.33  \\
			eil51     & 426    & 482    & 13.15  & 472    & 10.80 & 483    & 13.38  \\
			eil76     & 538    & 608    & 13.01  & 598    & 11.15 & 576    & 7.06   \\
			fl1400    & 20127  & 25115  & 24.78  & 24719  & 22.82 & 24587  & 22.16  \\
			fl417     & 11861  & 13887  & 17.08  & 13866  & 16.90 & 13581  & 14.50  \\
			fri26     & 937    & 965    & 2.99   & 959    & 2.35  & 960    & 2.45   \\
			gil262    & 2378   & 2823   & 18.71  & 2767   & 16.36 & 2768   & 16.40  \\
			gr120     & 6942   & 8438   & 21.55  & 8335   & 20.07 & 8411   & 21.16  \\
			gr17      & 2085   & 2178   & 4.46   & 2178   & 4.46  & 2178   & 4.46   \\
			gr21      & 2707   & 3003   & 10.93  & 2958   & 9.27  & 2998   & 10.75  \\
			gr24      & 1272   & 1553   & 22.09  & 1400   & 10.06 & 1476   & 16.04  \\
			gr48      & 5046   & 5840   & 15.74  & 5561   & 10.21 & 5695   & 12.86  \\
			hk48      & 11461  & 12137  & 5.90   & 12031  & 4.97  & 11990  & 4.62   \\
			kroA100   & 21282  & 24698  & 16.05  & 24582  & 15.51 & 24548  & 15.35  \\
			kroA150   & 26524  & 31479  & 18.68  & 31320  & 18.08 & 31234  & 17.76  \\
			kroA200   & 29368  & 34543  & 17.62  & 34543  & 17.62 & 35329  & 20.30  \\
			kroB100   & 22141  & 25884  & 16.91  & 25255  & 14.06 & 25546  & 15.38  \\
			kroB150   & 26130  & 31611  & 20.98  & 31524  & 20.64 & 30043  & 14.98  \\
			kroB200   & 29437  & 35389  & 20.22  & 35283  & 19.86 & 35454  & 20.44  \\
			kroC100   & 20749  & 23660  & 14.03  & 23603  & 13.75 & 23970  & 15.52  \\
			kroD100   & 21294  & 24852  & 16.71  & 24603  & 15.54 & 23722  & 11.40  \\
			kroE100   & 22068  & 24782  & 12.30  & 24445  & 10.77 & 24185  & 9.59   \\
			lin105    & 14379  & 16935  & 17.78  & 16147  & 12.30 & 15878  & 10.42  \\
			lin318    & 42029  & 49201  & 17.06  & 49201  & 17.06 & 48996  & 16.58  \\
			linhp318  & 41345  & 49201  & 19.00  & 49201  & 19.00 & 48996  & 18.51  \\
			nrw1379   & 56638  & 68531  & 21.00  & 67873  & 19.84 & 67415  & 19.03  \\
			p654      & 34643  & 43027  & 24.20  & 42935  & 23.94 & 42493  & 22.66  \\
			pa561     & 2763   & 3279   & 18.68  & 3269   & 18.31 & 3284   & 18.86  \\
			pcb1173   & 56892  & 70115  & 23.24  & 69085  & 21.43 & 69325  & 21.85  \\
			pcb442    & 50778  & 58950  & 16.09  & 58682  & 15.57 & 58599  & 15.40  \\
			pr76      & 108159 & 130921 & 21.04  & 128749 & 19.04 & 129467 & 19.70  \\
			si1032    & 92650  & 94083  & 1.55   & 93981  & 1.44  & 93731  & 1.17   \\
			si175     & 21407  & 22000  & 2.77   & 21906  & 2.33  & 21927  & 2.43   \\
			si535     & 48450  & 50036  & 3.27   & 50032  & 3.27  & 49853  & 2.90   \\
			swiss42   & 1273   & 1437   & 12.88  & 1425   & 11.94 & 1350   & 6.05   \\
			\bottomrule
		\end{tabular}
	\end{minipage}
	\caption{Results for 48 instances of the Symmetric TSP taken from $[$8$]$. The optimum and the result are given in absolute values. The excess represents the percentage by which the result exceeds the optimum.}
	\label{fig:STSP}
\end{figure}

Figure \ref{fig:STSP} and Figure \ref{fig:ATSP} show the results of experiments conducted for 48 instances of the STSP and 18 instances of ATSP taken from TSPLIB $[$8$]$.
It can be observed that the results of $2$-RNN are only slightly better than those of $1$-RNN, sometimes they are even equal.
We would like to point out two specific instances, the first is \texttt{brg180}.
For this 180-city instance, standard 1-RNN returns the poor result of 8890 (355.9\% above the optimum) while 2-RNN is able to avoid this returning a tour with length 2020 (3.59\% above the optimum).
For one other instance (\texttt{br17}), $2$-RNN even produces the exact result.

In general all results produced by $2$-RNN are reasonable. 
On average, the tours produced are 10 \% to 40 \% longer than the optimum.
Using these observations, we expect the algorithm to perform similarly for other instances although there might be instances were the output is of unreasonable quality.

Considering that a $2$-RNN run for an instance with 1379 vertices (namely TSPLIB's \texttt{nrw1379}) takes about 60 minutes to finish while a $1$-RNN run for the same instance only takes about 3 seconds (tested on a standard laptop), the quality of the solution does not increase by a big enough amount to justify the running time.

The results of the bidirectional variant of 2-RNN are of different quality. For some instances they outperform $2$-RNN, for some others $2$-RNN performs better.
As the bidirectional version tries to extend the tour in both directions rather than only one, the running time is about twice as long as for the normal variant.

\begin{figure}[!htbp]
	\centering
	\begin{minipage}{\linewidth}
		\centering
		\scriptsize
		\begin{tabular}{*{3}{L{1.6cm}} >{\bfseries}L{1.6cm} L{1.6cm} >{\bfseries}L{1.6cm} L{1.6cm} >{\bfseries}L{1.6cm}}
			\toprule
			&& \multicolumn{2}{c}{1-RNN}& \multicolumn{2}{c}{2-RNN} & \multicolumn{2}{c}{Bi-2-RNN}\\
			\cmidrule(lr){3-4}
			\cmidrule(lr){5-6}
			\cmidrule(lr){7-8}
			Dataset & Optimum & Result & Excess & Result & Excess & Result & Excess \\
			\midrule
			br17    & 39    & 56    & 43.59 & 39    & 0.00  & 56    & 43.59 \\
			ft53    & 6905  & 8584  & 24.32 & 8323  & 20.54 & 8757  & 26.82 \\
			ft70    & 38673 & 41815 & 8.12  & 41633 & 7.65  & 41847 & 8.21  \\
			ftv33   & 1286  & 1590  & 23.64 & 1544  & 20.06 & 1396  & 8.55  \\
			ftv35   & 1473  & 1667  & 13.17 & 1606  & 9.03  & 1667  & 13.17 \\
			ftv38   & 1530  & 1759  & 14.97 & 1709  & 11.70 & 1792  & 17.12 \\
			ftv44   & 1613  & 1844  & 14.32 & 1829  & 13.39 & 1833  & 13.64 \\
			ftv47   & 1776  & 2173  & 22.35 & 2149  & 21.00 & 2115  & 19.09 \\
			ftv55   & 1608  & 1948  & 21.14 & 1854  & 15.30 & 1848  & 14.93 \\
			ftv64   & 1839  & 2202  & 19.74 & 2202  & 19.74 & 2176  & 18.33 \\
			ftv70   & 1950  & 2287  & 17.28 & 2218  & 13.74 & 2261  & 15.95 \\
			ftv170  & 2755  & 3582  & 30.02 & 3559  & 29.18 & 3334  & 21.02 \\
			kro124p & 36230 & 43316 & 19.56 & 43102 & 18.97 & 41562 & 14.72 \\
			p43     & 5620  & 5684  & 1.14  & 5653  & 0.59  & 5639  & 0.34  \\
			rbg323  & 1326  & 1702  & 28.36 & 1684  & 27.00 & 1743  & 31.45 \\
			rbg358  & 1163  & 1747  & 50.21 & 1719  & 47.81 & 1570  & 35.00 \\
			rbg403  & 2465  & 3497  & 41.87 & 3460  & 40.37 & 2928  & 18.78 \\
			ry48p   & 14422 & 15575 & 7.99  & 15575 & 7.99  & 15308 & 6.14  \\
			\bottomrule
		\end{tabular}
	\end{minipage}
	
	\caption{Results for 18 instances of the Asymmetric TSP taken from TSPLIB $[$8$]$. The optimum and the result are given in absolute values. The excess represents the percentage by which the result exceeds the optimum.}
	\label{fig:ATSP}
\end{figure}

Figure \ref{fig:ATSP} shows experimental results for 18 instances of the ATSP.
In comparison to the $2$-RNN results for the STSP, the $2$-RNN results here seem to be slightly worse:
Whereas there is no instance where the tour length exceeds the optimum by more than 25\% for STSP there are 4 instances for the ATSP where this is the case.

	\section{Conclusion}
	\label{sec:conclusion}
	
	We presented an extension of the already known Nearest-Neighbor heuristic to a family of algorithms we call $k$-Repetitive-Nearest-Neighbor ($k$-RNN). 
	This algorithm takes all permutations of $k$ vertices as a starting tour and performs a Nearest-Neighbor search from there on. 
	The result is the shortest of all tours found that way.
	
	We have proven that as $k$ increases, so does the quality of the tours found.
	Despite this, our experimental results only show a slight increase in the quality of solution as $k$ increases, meanwhile the running time increases by a factor of $n$.
	In one case a larger $k$ ($2$-RNN) was able to avoid an undesirable result from $1$-RNN, reducing the excess by over 350\%.

	A scope of future research could be a more thorough theoretical analysis of the algorithm, especially an extension of the domination analysis to the more general $k$-RNN will give more insight in its competitiveness among other algorithms. Also, experiments on more varied and larger instances are desirable.


\end{document}